\documentclass[letterpaper, 10 pt, conference]{ieeeconf}  
\usepackage{amssymb}  
\usepackage{bm}
\usepackage{mathtools}
\usepackage{xcolor}
\usepackage{graphicx}
\usepackage{booktabs}
\usepackage{cite}
\usepackage{url}
\usepackage{hyperref}
\usepackage{todonotes}
\let\labelindent\relax
\usepackage{enumitem}
\usepackage{multirow}
\usepackage{algorithm}
\usepackage{algpseudocode}
\usepackage{stfloats}

\usepackage{fontawesome5}
\usepackage[most]{tcolorbox}
\newtcbox{\linkpill}{
  on line,
  box align=base,
  colback=gray!10,
  colframe=gray!10,
  arc=1.5mm,
  boxrule=0pt,
  left=2pt,
  right=2pt,
  top=1pt,
  bottom=1pt
}
\definecolor{slsorange}{HTML}{D97706}
\definecolor{slsnavy}{HTML}{174A7E}

\newcommand{\R}{\mathbb{R}}
\newcommand{\nx}{{n_x}}

\newcommand{\ntheta}{{n_{\theta}}}
\newcommand{\nuu}{{n_u}}
\newcommand{\Phix}{\mathbf{\Phi}^{\mathrm{x}}}
\newcommand{\Phis}{\mathbf{\Phi}^{\mathrm{s}}}
\newcommand{\Phiu}{\mathbf{\Phi}^{\mathrm{u}}}
\newcommand{\Phit}{\mathbf{\Phi}}

\newcommand{\Psit}{\mathbf{\Psi}}

\newcommand{\BPhit}{\mathbf{\bar\Phi}}

\newcommand{\TPhit}{\mathbf{\tilde\Phi}}
\newcommand{\Mparam}[1]{\bm{M}_{#1,\xi}}
\newcommand{\Mdist}[1]{\bm{M}_{#1}}

\newtheorem{theorem}{Theorem}

\newtheorem{remark}{Remark}

\newtheorem{proposition}{Proposition}
\newtheorem{corollary}{Corollary}

\newcommand{\delete}[1]{}

\IEEEoverridecommandlockouts                              

\usepackage{graphicx} 

\title{\LARGE \bf
Planning to Learn: Real-Time Robust Adaptive Control on the GPU with Guaranteed Future Learning
}

\author{Anutam Srinivasan$^{1}$\and Siddhant Panse$^{1}$\and Glen Chou$^{2}$
\thanks{$^{1}$Georgia Institute of Technology, School of Electrical \& Computer Engineering, \texttt{\{\href{mailto:asrinivasan350@gatech.edu}{asrinivasan350},\href{mailto:spanse30@gatech.edu}{spanse30}\}@gatech.edu}}%
\thanks{$^{2}$Georgia Institute of Technology, Schools of Cybersecurity \& Privacy, Aerospace Engineering, \texttt{\href{mailto:chou@gatech.edu}{chou@gatech.edu}}}%
}

\begin{document}

\maketitle
\thispagestyle{empty}
\pagestyle{empty}

\begin{abstract}
Robust adaptive control is vital for keeping robots safe while actively reducing uncertainty in unfamiliar settings. However, traditional methods are often conservative since they fail to account for guaranteed future parameter learning during planning and struggle to scale to high-dimensional systems. In this work, we present A-SLS, a learning-aware robust adaptive model predictive control (RAMPC) framework based on system-level synthesis. A-SLS forecasts future parameter uncertainty and preserves correlations between process disturbances, parameter error, and estimate updates to construct tight, guaranteed reachable sets that account for future information gain. We also develop a GPU-accelerated A-SLS solver to enable real-time performance. Across nonlinear robotic systems, A-SLS reduces reachable-tube conservativeness, maintains real-time performance, and scales to whole-body control for a quadruped with 61 states, 8 uncertain parameters, and 12 control inputs. We further demonstrate the approach on hardware, where a physical Crazyflie drone learns an uncertain dynamics parameter online to stabilize on the fly.
\end{abstract}
\begin{center}
\vspace{-2.5pt}
{\linkpill{\href{https://trustworthyrobotics.github.io/adaptive_sls/}{\textcolor{slsnavy}{\faGlobe\ \texttt{Website}}}}
\quad
\linkpill{\href{https://github.com/trustworthyrobotics/adaptive_sls}{\textcolor{slsnavy}{\faGithub\ \texttt{Code}}}}\quad
\linkpill{\href{https://youtu.be/9bq6-69Yino}{\textcolor{slsnavy}{\faYoutube\ \texttt{Video}}}}}
\vspace{-2.5pt}
\end{center}

\begin{figure}[!b]
    \centering
    \includegraphics[width=\linewidth]{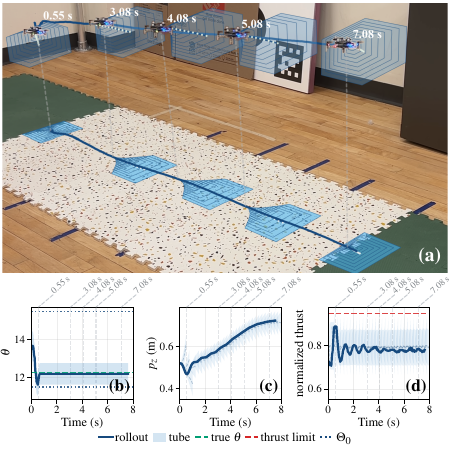}\vspace{-25pt}
    \caption{\looseness-1\textbf{Hardware Crazyflie drone.} An example rollout of the Crazyflie learning an unknown thrust-to-acceleration parameter. (a) Rollout with forecasted tubes projected at selected timestamps. (b) Parameter uncertainty tubes and reduction over time. (c) The $p_z$ position and corresponding tubes over time. (d) Thrust and corresponding tubes with the thrust limit.}
    \label{fig:front_figure_quadrotor_hardware.pdf}
\end{figure}
\section{Introduction}

Robots operating in unstructured environments must act safely despite uncertainty in their dynamics and surroundings. Robust control is one approach to achieve this, providing worst-case guarantees by treating uncertainty as persistently adversarial, propagating reachable tubes over admissible uncertainty, and tightening constraints accordingly. While this is appropriate for irreducible \emph{aleatoric} uncertainty, it can be overly conservative for \emph{epistemic} uncertainty arising from unknown but learnable physical parameters such as mass, friction, wind, or motor damage. Because this structured uncertainty can be reduced through interaction, robust adaptive control can become less conservative as parameter estimates improve online. As data is collected and parameter uncertainty shrinks, controllers can safely exploit the reduced uncertainty to improve performance. Moreover, because the rate of learning depends on how informative the chosen trajectory is, control actions can directly influence future parameter uncertainty, creating a tradeoff between performance, robustness, and information gathering.

\begin{figure}[!b]
    \centering\vspace{-10pt}
    \includegraphics[width=\linewidth]{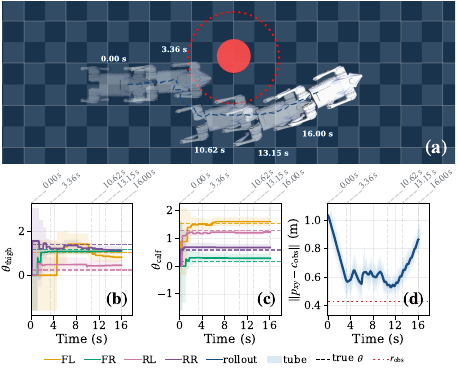}\vspace{-21pt}
    \caption{\looseness-1\textbf{61D quadruped, whole-body control.} (a) A sample A-SLS quadruped rollout while learning 8 joint-damping parameters. (b-c) Subset of parameter tubes: learned thigh/calf parameters and uncertainty tubes for the front left (FL), front right (FR), rear left (RL), and rear right (RR) legs. (d) Obstacle proximity/projected tube over time. Using the learned parameters, A-SLS stabilizes the quadruped while maintaining safe obstacle clearance.}
    \label{fig:quadruped_rollout_figure}
\end{figure}

However, designing robust adaptive controllers for robotics remains an open problem, with two central challenges. First, synthesizing adaptive controllers while jointly reasoning about estimation error, aleatoric uncertainty, and parametric uncertainty is computationally expensive for long-horizon, high-dimensional nonlinear systems, yet practical robotic systems require real-time learning, verification, and controller adaptation. Second, many existing methods update parameter estimates online but plan using the \emph{current} parameter uncertainty set, incorporating newly learned information only \emph{reactively} after measurements are collected. As a result, they cannot exploit guaranteed \emph{future} uncertainty reduction during planning and may produce unnecessarily large reachable tubes even along informative trajectories. For example, a drone traversing an uncertain but learnable wind field should be able to plan knowing that it can measure, estimate, and compensate for the wind during execution. Ignoring this future learning can degrade performance or prevent task completion, while also leading to myopic planning that fails to seek informative trajectories that accelerate uncertainty reduction. Addressing these challenges requires jointly modeling the state and parameter-estimate dynamics while efficiently forecasting the worst-case learning guaranteed under all admissible uncertainties.

\looseness-1In this work, we close these gaps with A-SLS, a scalable robust adaptive policy optimization framework that explicitly accounts for future learning during planning. Building on system-level synthesis (SLS) \cite{anderson2019system}, we augment the physical state with the online parameter estimate so that estimator evolution becomes part of the closed-loop dynamics. We propagate bounded parameter uncertainty and parameter-estimate shifts over the planning horizon, while adaptive SLS responses preserve correlations among process disturbances, parameter error, and estimate updates to construct tighter reachable tubes. Crucially, this enables \emph{assured future learning}: for each candidate trajectory, we forecast the parameter uncertainty guaranteed to remain after future observations, allowing the optimizer to exploit worst-case future tube contraction without assuming favorable measurements. Because informativeness depends on the trajectory, the same formulation can also favor motions that accelerate uncertainty reduction. Finally, the structured SLS policy class admits efficient parallel reachability and robust constraint tightening, enabling GPU-parallel optimization for long-horizon, high-dimensional nonlinear systems. Our contributions are:
\begin{itemize}
    \item We formulate learning-aware robust adaptive control as an SLS problem that augments the state with the online parameter estimate, returning policies that adapt to online data while remaining robust to residual uncertainty.
    \item We derive reachable sets that preserve correlations between disturbances, parameter error, and estimate shifts, enabling tightened robust tubes that account for and guarantee future learning. These information-sensitive tubes also enable planning that favors informative trajectories that reduce uncertainty and improve performance.
    \item We develop a GPU-parallel solver that enables real-time long-horizon control on high-DoF nonlinear systems.
    \item We evaluate safety, tube tightness, parameter learning, active information gathering, and scalability on high-dimensional nonlinear robotic systems and hardware.
\end{itemize}

\section{Related Work}

Classic robust adaptive control~\cite{ioannou1984robust,ioannou1996robust,polycarpou1996robust,astrom1995adaptive,cao2008l1} establishes stability guarantees under disturbances and parametric uncertainty. Recent work extends these ideas to nonlinear and unmatched uncertainty using contraction theory~\cite{lopez2019contraction,lopez2020universal,tsukamoto2021learning,gessow2023adaptive} and $\mathcal{L}_1$-LPV control~\cite{zhao2024lpv}. Other approaches develop SLS-based robust control under parametric uncertainty~\cite{leeman2023robust} and robust adaptive control for linear systems~\cite{ho2019scalable}. While these methods provide strong stability guarantees, they generally do not optimize long-horizon constrained trajectories or anticipate how planned actions reduce uncertainty; contraction-based approaches also require challenging metric synthesis and online geodesic computation. In contrast, our method uses SLS for robust adaptive nonlinear trajectory optimization, coupling closed-loop reachability with guaranteed future learning and real-time GPU-parallel synthesis.

Other methods incorporate adaptation into safety-critical control. Adaptive and robust-adaptive control barrier functions (CBFs) provide forward-invariance guarantees under parametric uncertainty~\cite{taylor2020adaptive,lopez2021racbf,xiao2022adaptive}, with later applications to robotics~\cite{minniti2022adaptive,sinaei2026safe}. Many reduce online control to small QPs or impose Lyapunov conditions within MPC. However, synthesizing adaptive CBFs/CLFs can be challenging, and these methods typically provide pointwise invariance or stability guarantees without explicitly planning trajectories to improve parameter learning. Our approach instead couples closed-loop reachability and planning through shared uncertainty sources affecting physical states and future parameter estimates, yielding tighter robust tubes while making information gathering part of the trajectory optimization.

Robust adaptive model predictive control (RAMPC) combines online model learning with constrained planning. Early work used conservative constraint tightenings or separate robust safety and learned performance models~\cite{adetola2011robust,aswani2013provably}. Set-membership identification was later integrated with robust MPC for linear~\cite{lorenzen2019robust} and nonlinear systems~\cite{kohler2021robust,sinha2022matched}, with extensions addressing convergence~\cite{lu2021performance}, safety certification~\cite{didier2021adaptive}, $\mathcal{L}_1$ adaptation~\cite{pereida2021robust}, contraction~\cite{sasfi2023robust}, and distributed control~\cite{faliero2026robust}. However, most methods exploit learning only after uncertainty updates are observed rather than anticipating guaranteed future estimator evolution, and can be computationally demanding online. Our adaptive SLS formulation instead jointly propagates parameter-estimation and disturbance dynamics, preserves their correlations for tighter tubes, and enables real-time GPU-parallel synthesis. 

Finally, dual and active-learning MPC recognize that control affects both task execution and future model quality~\cite{mesbah2018stochastic}. Persistent-excitation constraints can promote parameter convergence in RAMPC~\cite{lu2023persistent}, while~\cite{parsi2023dual} embeds predicted set-membership updates in linear MPC for informative exploration. These methods closely align with our planning-to-learn objective. However, persistent excitation is an indirect proxy for task-relevant uncertainty reduction, while exact future set-membership reasoning can be expensive and difficult to scale. Our method instead optimizes worst-case parameter-estimation dynamics within a scalable SLS framework, exploiting assured future tube contraction while retaining robust nonlinear constraints and GPU-parallel computation for long-horizon, high-dimensional systems.

\section{Preliminaries and Problem Statement}
\looseness-1\noindent\textbf{Notation:} The symbols \(I_n \in \mathbb{R}^{n\times n}\), \(0_{n\times m} \in \mathbb{R}^{n\times m}\), and \(0_n := 0_{n \times 1}\) denote the identity matrix, zero matrix, and zero vector, respectively. The $i$-th standard coordinate row vector is denoted by $\bar{e}_i^{(n)}\in\mathbb{R}^{1 \times n}$, whose $i$-th entry is $1$ and all other entries are $0$. The Frobenius norm of a matrix \(A \in \mathbb{R}^{m \times n}\) is given by \(\| A \|_{\mathcal{F}} := \sqrt{\text{Trace}(A^\top A)}\). For $A \in \mathbb{R}^{n\times m}$, $[A]_{k,j}$ denotes the element in the $k$-th row and $j$-th column of $A$. Similarly, we use \([\mathbf{\Phi}]_{k,j}\) to refer to the block at the \(k\)-th row and \(j\)-th column of a block matrix \(\mathbf{\Phi}\). For integers \(M,N \in \mathbb{N}\), index sets are defined as \([N] := \{0,\ldots,N-1\}\) and \([M,N] := \{M,\ldots,N\}\). For a matrix product, indices increase from right to left, i.e., $\prod_{m=\ell}^{j-1} A_m := A_{j-1}\cdots A_\ell$. For a matrix \(A \in \mathbb{R}^{n \times p}\), its row-wise \(\ell_1\)-norm is the vector \(\| A \|_{1,\text{r}} \coloneqq [\|A_{1,:}\|_1, \ldots, \|A_{n,:}\|_1]^\top \in \mathbb{R}^{n}\). The diagonal matrix constructed from \(a \in \mathbb{R}^n\) is written as \(\text{diag}(a)\). 
The Minkowski sum of sets \(A\) and \(B\) is defined by \(A \oplus B := \{a + b \mid a \in A, b \in B\}\) and $\bigoplus$ denotes the iterated Minkowski sum. We define the unit $\ell_\infty$-ball as $\mathcal{B}^{n}_\infty:=\{z\in\mathbb{R}^{n}\mid\|z\|_\infty \le 1\}$. 

\noindent\textbf{Definitions:} We consider nonlinear dynamics with both parametric and exogenous worst-case uncertainty,
\begin{equation}\label{eq:uncert_param_system}
x_{k+1} = f(x_k,u_k;\theta) + F(x_k)w_k,
\end{equation}
where $x_k \in \mathcal{X} \subseteq \R^\nx$ is the system state, $u_k \in \mathcal{U}\subseteq \R^\nuu$ is the control input, and $\theta \in \Theta_0 \subseteq\R^\ntheta$ is an unknown parameter (e.g., mass) contained in an initial zonotopic uncertainty set $\Theta_0$. The map $f:\R^\nx\times\R^\nuu\times\R^\ntheta \to \mathbb{R}^{\nx}$ defines the nonlinear dynamics, while $F:\mathbb{R}^{\nx} \to \mathbb{R}^{\nx \times \nx}$ scales the bounded disturbance $w_k \in  \mathcal{B}^{n_x}_\infty$. Because $\theta$ is unknown at planning time, the controller begins from an initial estimate $\hat{\theta}_0$ and refines it online as measurements are collected, yielding $\hat{\theta}_k$ at time $k$. We seek a robust-adaptive policy with (i) a feedback component that tracks the planned state trajectory and (ii) a feedforward component that exploits updated parameter estimates as uncertainty is reduced. The policy must robustly satisfy all state and input constraints despite residual parametric uncertainty and process noise. We therefore seek to solve the following robust-adaptive nonlinear optimal control problem (RANOCP):
\begin{subequations} \label{eq:ranocp}
    \begin{align}
        \hspace{-20mm}\min_{\pi(\cdot)} \ &J\left(x_1,\dots,x_N, \pi(\cdot)\right) \\
        \text{s.t.} \ & x_{k+1} = f(x_k, u_k;\hat\theta_k) + E(x_k,u_k,\hat\theta_k;\theta) \nonumber\\ &\qquad\qquad+ F(x_k)w_k, \quad \forall k \in [N],
        \\
                          & x_0 = \bar{x}_0, \quad\hat\theta_0 = \hat\theta,\\ 
                          & \hat\theta_k = m_k(\hat\theta_0, r_1, \dots, r_k; \theta), \quad\forall k\in[N]\\
                          & u_k = \pi_k(x_{0},\dots,x_{k}, \hat\theta_0, \dots,\hat\theta_k), \quad \forall k \in [N], \\
                          & g_k(x_k, u_k) \leq 0,  \,\,\forall w_1,\dots,w_k \in \mathcal{B}^{n_x}_\infty,\nonumber\\
                          & \quad\qquad\qquad\qquad \forall \theta \in \Theta_k,\,\forall k \in [N], \label{eq:robust_nocp_con_running}\\
                          & g_N(x_N) \leq 0, \,\, \forall w_1,\dots,w_N \in \mathcal{B}^{n_x}_\infty,\,\, \forall \theta \in \Theta_N,\label{eq:robust_nocp_con_terminal}
    \end{align}
\end{subequations}
\looseness-1where $\pi:=\{\pi_i\}_{i=0}^{N-1}$ is a sequence of causal control policies, $m_k:\mathbb{R}^{\ntheta}\times(\mathbb{R}^{\nx})^k \times\mathbb{R}^{\ntheta}\to \mathbb{R}^{\ntheta}$ is a measurement function based on the history of residual measurements $r_k$, $\bar{x}_0\in\mathbb{R}^{\nx}$ is the initial state, and $\hat{\theta}, \theta\in\mathbb{R}^{\ntheta}$ are the initial parameter guess and true parameter, respectively. The set $\Theta_k$ denotes the parameter uncertainty set at step $k$, with $\hat{\theta}_k\in\Theta_k$ the parameter estimate. The functions $g_k:\mathbb{R}^{n_x}\times\mathbb{R}^{n_u}\rightarrow\mathbb{R}^{n_c}$ denote stagewise state-input constraints, while $g_N:\mathbb{R}^{n_x}\rightarrow\mathbb{R}^{n_c^f}$ denotes the terminal constraint. Lastly, $E:\mathbb{R}^{\nx}\times\mathbb{R}^{\nuu}\times\mathbb{R}^{\ntheta}\times\mathbb{R}^{\ntheta}\rightarrow\mathbb{R}^{n_x}$ represents the residual dynamics error from an incorrect parameter estimate.

Directly solving \eqref{eq:ranocp} is typically intractable~\cite{fang2026gpu}. Thus, we instead optimize a nominal state-control trajectory and robust tracking controller, perform closed-loop reachability around the trajectory, and enforce constraints over the resulting reachable set. We formulate this as an SLS problem~\cite{anderson2019system}, solved using GPUSLS~\cite{fang2026gpu} and GPUSLS-LEO~\cite{fang2026safe}.

\noindent\textbf{System Level Synthesis (SLS):} SLS is a robust controller synthesis framework. Specifically, SLS optimizes causal feedback gains to track a reference trajectory. For the following uncertain LTV system, 
\begin{equation}\label{eq:ltv_state}
    x_{k+1} = A_kx_k + B_ku_k + F_kw_k,
\end{equation}
the closed-loop control input is
\begin{equation}\label{eq:disturbance_feedback}
    u_k = v_k + \Delta u_k =  v_k + \textstyle\sum_{j=0}^{k-1} \Phiu_{k,j} F_jw_j,
\end{equation}
where $v_k\in\mathbb{R}^{\nuu}$ is the nominal control, $\Delta u_k$ is the feedback correction, and $\Phiu_{k,j}\in\mathbb{R}^{\nuu\times\nx}$ maps each past disturbance $F_jw_j$ to the control input.
Under this policy, the closed-loop state response is
\begin{equation}\label{eq:SLS_state}
x_k= z_k + \Delta x_k = z_k + \textstyle\sum_{j=0}^{k-1} \Phix_{k,j} F_jw_j, \qquad z_0 = \bar x_0,
\end{equation}
where $z_k\in\mathbb{R}^{n_x}$ is the nominal state, $\Delta x_k$ is the state deviation, and $\Phix_{k,j}\in\mathbb{R}^{\nx\times\nx}$ maps disturbance $F_jw_j$ to its effect on $x_k$. We set $\Phix_{j+1,j}=I_{\nx}$ and propagate $\Phix_{k+1,j}=A_k\Phix_{k,j}+B_k\Phiu_{k,j}$,
for all $j\in[N]$ and $k\in[j+1,N-1]$.
For nonlinear systems with nominal dynamics $z_{k+1}=f(z_k,v_k)$, we construct the LTV system in \eqref{eq:ltv_state} by linearizing about the nominal trajectory:
\begin{equation}\label{eq:linearizations}
\textstyle A_k := \frac{\partial f}{\partial x}\big|_{(z_k,v_k)}, \ \ 
B_k := \frac{\partial f}{\partial u}\big|_{(z_k,v_k)}, \ \ 
F_k := F(z_k).
\end{equation}
Taylor models~\cite{fang2026gpu} provide linearization error bounds (LEBs) that can be efficiently computed on the GPU \cite{shen2026parallel} and incorporated as an additional bounded disturbance by replacing $F_kw_k$ with $F_kw_k + F^\textrm{nl}_kw^\textrm{nl}_k$, where $F^\textrm{nl}_kw^\textrm{nl}_k$ captures the linearization error, with $w^\textrm{nl}_k\in\mathcal{B}^{\nx}_\infty$~\cite{fang2026safe}.

\noindent\textbf{Problem Statement (Learning-Aware Robust Adaptive Control):}
Given the uncertain nonlinear system \eqref{eq:uncert_param_system}, initial state $\bar{x}_0$, parameter estimate $\hat{\theta}_0$, uncertainty set $\Theta_0$, bounded disturbance $w \in \mathcal{B}^{\nx}_\infty$, horizon $N$, and constraints $g_k$ and $g_N$, our goal is to solve the RANOCP~\eqref{eq:ranocp} by jointly planning a nominal state-control trajectory and synthesizing a causal adaptive policy $\pi$. Using observed transition residuals $r_k$, we update the parameter estimate $\hat{\theta}_k$ online while bounding the future uncertainty sets $\Theta_k$ \emph{a priori}. The policy must minimize the control objective while robustly satisfying all constraints for every admissible parameter and disturbance realization. Notably, planning accounts for the worst-case future parameter uncertainty induced by measurements along the planned trajectory, allowing predicted reachable tubes to exploit guaranteed future learning rather than treating the current parameter uncertainty set as fixed over the horizon.


\section{Methodology}
\looseness-1In this section, we present the components of our adaptive SLS formulation (A-SLS) used to plan with guaranteed future learning. We first derive the parameter-estimation (Sec. \ref{subsec:method_parameter_estimation}) and uncertainty-update dynamics (Sec. \ref{subsec:bounded_update}), then embed them into SLS to construct correlation-preserving reachable sets (Sec. \ref{subsec:method_sls}) and robust constraint tightenings (Sec. \ref{subsec:robust_constraints}). We then extend to nonlinear dynamics and discuss the resulting robust-adaptive optimization problem (Sec. \ref{subsec:overall_optimization}). 

\subsection{Parameter Estimation Dynamics}\label{subsec:method_parameter_estimation}

We begin with the parametric LTV dynamics
\begin{equation}\label{eq:ltv_parametric}
    x_{t+1} = A_tx_t + B_tu_t + E_t\theta + F_tw_t,
\end{equation}
where $E_t \in\mathbb{R}^{\nx\times\ntheta}$ is the parameter sensitivity matrix describing how $\theta$ affects the dynamics.
To adapt the controller, we estimate unknown parameters from observed state transitions. Given the applied control $u_t$ and measured successor state $x_{t+1}$, the only unknown quantities in \eqref{eq:ltv_parametric} are $\theta$ and $w_t$. To mitigate the effects of noisy measurements on estimation accuracy, we introduce a binary diagonal measurement gate $C_t \in \mathbb{R}^{\nx\times\nx}$ that selects the state measurement channels used for parameter estimation. We then define 
\begin{equation}\label{eq:gated_information}
H_t:=C_tE_t = U_t \Sigma_t V_t^\top,
\quad
H_{t,\epsilon}^{\dagger}
=
V_t \Sigma_{t,\epsilon}^{\dagger} U_t^\top,
\end{equation}
where $H_t\in \mathbb{R}^{\nx\times\ntheta}$ is the gated parameter sensitivity and $U_t\Sigma_tV_t^\top$ is its singular value decomposition (SVD). The SVD identifies which directions in parameter space are observable through the selected measurements. However, inverting small singular values can amplify disturbance noise. To avoid this, we define the truncated pseudoinverse:
\begin{equation}
[\Sigma_{t,\epsilon}^{\dagger}]_{ii}
=
\begin{cases}
\sigma_{t,i}^{-1}, & \sigma_{t,i}\geq\epsilon_t,\\
0,                 & \sigma_{t,i}<\epsilon_t,
\end{cases}
\,
\epsilon_t=\max\{\epsilon_{\mathrm{abs}},
                 \epsilon_{\mathrm{rel}}\sigma_{t,1}\},
\end{equation}
where $\epsilon_t$, $\epsilon_{\mathrm{abs}}$, and $\epsilon_{\mathrm{rel}}$ determine which singular values are inverted~\cite{golub1965calculating,golub2013matrix}. This allows parameter estimation even when some components of $\theta$ are unobservable through $H_t$.
We then define the gated measurement map $\Pi_t\in\mathbb{R}^{\ntheta\times\ntheta}$ and measurement uncertainty map $W_t\in\mathbb{R}^{\ntheta\times\nx}$:
\begin{equation}\label{eq:param_uncertainty_definition}
\Pi_t := H_{t,\epsilon}^{\dagger}H_t,
\,\quad
W_t:=H_{t,\epsilon}^{\dagger}C_tF_t.
\end{equation}
Here, $\Pi_t$ projects onto the retained observable parameter directions, while $W_t$ captures how disturbances propagate into the parameter measurement. Thus, applying the truncated pseudoinverse to the gated transition residual $r_t$ yields 
\begin{equation}\label{eq:parameter_measure}
\tilde{\theta}_{t+1}
=
H_{t,\epsilon}^{\dagger}\underbrace{C_t
\left(x_{t+1}-A_tx_t-B_tu_t\right)}_{r_t}
=
\Pi_t\theta + W_tw_t.
\end{equation}
Hence, $\tilde{\theta}_{t+1}$ provides a noisy measurement of the observable component of $\theta$, which we use to update the parameter estimate and its uncertainty set.

\subsection{Bounded Set Update}\label{subsec:bounded_update}

\looseness-1Using the parameter measurement model of Sec. \ref{subsec:method_parameter_estimation}, we now develop a set-valued estimation update that propagates bounded parameter uncertainty while accounting for newly observed information. 
Since the parameter-estimation error evolves affinely under bounded disturbances and zonotopes are closed under linear transformations and Minkowski sums, the initially zonotopic parameter uncertainty set $\Theta_0$ admits a zonotopic outer approximation at every future timestep. Specifically, given the uncertainty set $\Theta_t := \hat\theta_t \oplus G_t\mathcal{B}^{p_t}_\infty$, where $\hat\theta_t$ is the current parameter estimate and $G_t$ contains $p_t$ generators, we update the zonotope center and generators as
\begin{align}
    \hat{\theta}_{t+1}
    &= \hat{\theta}_{t}
    +K_{t}\left(\tilde{\theta}_{t+1}-\Pi_{t}\hat{\theta}_{t}\right)
    \label{eq:zonotope_center_update} \\
    G_{t+1}
    &= \begin{bmatrix}
    (I_{\ntheta}-K_{t}\Pi_{t})G_{t} & K_{t}W_{t}
    \end{bmatrix},
    \label{eq:zonotope_gen_update}
\end{align}
where $K_t \in \mathbb{R}^{\ntheta\times\ntheta}$ is the Kalman gain. The center update incorporates newly observed parameter information, while the generator update propagates residual parameter uncertainty and measurement disturbances and is efficient to compute online, ensuring that $\Theta_{t+1}$ contains the true parameter, i.e., $\theta \in \Theta_{t+1}$, for all admissible disturbances \cite{alamo2005guaranteed}. 
Since $G_{t+1}$ determines the parameter uncertainty in $\Theta_{t+1}$, we minimize the Frobenius norm of \eqref{eq:zonotope_gen_update} as a surrogate for zonotope volume by choosing
\begin{equation}\label{eq:gain_update_law}
    K_{t}
    =G_{t}G_{t}^\top\Pi_{t}^\top
    \left(
    \Pi_{t}G_{t}G_{t}^\top\Pi_{t}^\top
    +W_{t}W_{t}^\top
    \right)^{\dagger}.
\end{equation}
Since each update in \eqref{eq:zonotope_gen_update} adds $n_x$ additional disturbance generators, the number of generators grows by $n_x$ at each stage. To maintain tractability, Girard reduction can replace the resulting zonotope with a lower-order outer approximation, limiting memory growth~\cite{girard2005reachability}. To prevent the uncertainty bound from increasing in any parameter dimension, we first compute $K_t$ from \eqref{eq:gain_update_law} and compare the row-wise $\ell_1$ norms $\|G_{t+1}\|_{1,\mathrm{r}}$ and $\|G_t\|_{1,\mathrm{r}}$, which give the uncertainty radii for individual parameters. If an entry increases, we set the corresponding row of $K_t$ to zero and recompute the update. This leaves that dimension's uncertainty radius unchanged, guaranteeing nonincreasing coordinate-wise uncertainty while preserving containment of the true parameter. The resulting parameter uncertainty set is $\Theta_{t+1}=\hat\theta_{t+1}\oplus G_{t+1}\mathcal{B}^{p_{t+1}}_\infty$. 
Although the future set center $\hat\theta_{t+1}$ is unknown at planning time, $W_t$ in \eqref{eq:param_uncertainty_definition} can be computed \emph{a priori} along a candidate trajectory. This enables us to evaluate $K_t$ using \eqref{eq:gain_update_law} and recursively propagate \eqref{eq:zonotope_gen_update} to forecast $G_t$ for all $t\in[N]$. Consequently, the parameter-estimation error satisfies
\begin{equation}\label{eq:param_error_w_bound}
e_t := \theta-\hat\theta_t \in G_t\mathcal{B}^{p_t}_\infty,
\end{equation}
where the bounding zonotope is known \emph{a priori}, without requiring future measurements. This allows the controller to anticipate guaranteed future parameter uncertainty reduction during planning. Similarly, using \eqref{eq:zonotope_center_update}-\eqref{eq:zonotope_gen_update}, we can bound the parameter-estimate shift as follows: 
\begin{subequations}\label{eq:param_shift_w_bound}
\begin{align}
\Delta\hat\theta_t
&:= \hat\theta_{t+1}-\hat\theta_t
\label{eq:param_shift}\\
&\in
\begin{bmatrix}
K_t\Pi_tG_t & K_tW_t
\end{bmatrix}
\mathcal{B}_\infty^{p_t+n_x},
\label{eq:param_shift_bound}
\end{align}
\end{subequations}
whose bounding zonotope is also known during planning.

\subsection{System Level Synthesis Formulation}\label{subsec:method_sls}
\looseness-1To compute a robust adaptive controller responding to both state disturbances and updated parameter estimates, we augment the state with the parameter estimate. We rewrite \eqref{eq:ltv_parametric} as
\begin{equation} \label{eq:ltv_with_estimate}
x_{t+1} = A_tx_t + B_tu_t + E_t\theta - E_t\hat\theta_t + E_t\hat\theta_t + F_tw_t,
\end{equation}
while \eqref{eq:param_shift} gives the parameter-estimate dynamics
\begin{equation}\label{eq:theta_dynamics}
\hat\theta_{t+1} = I_{\ntheta}\hat\theta_t + I_{\ntheta}\Delta\hat\theta_t.
\end{equation}
Defining
\begin{equation}\label{eq:augment_terms}
\begin{aligned}
\bar{A}_t&:=\begin{bmatrix}
A_t & E_t \\ 0_{n_\theta \times n_x} & I_{n_\theta}
\end{bmatrix}, \quad
\bar{B}_t:=\begin{bmatrix}
B_t \\ 0_{n_\theta \times n_u}
\end{bmatrix},\\
\bar{F}_t&:= \begin{bmatrix}
F_t & E_t & 0_{n_x\times n_\theta} \\
0_{n_\theta \times n_x} & 0_{n_\theta \times n_\theta} & I_{n_\theta}
\end{bmatrix}, \quad
s_t:=\begin{bmatrix}
x_t \\ \hat\theta_t
\end{bmatrix},
\end{aligned}
\end{equation}
we construct augmented dynamics that include bounded but unknown changes in parameter estimates, resembling \eqref{eq:ltv_state}:
\begin{equation}\label{eq:augmented_system}
s_{t+1} = \bar{A}_ts_t + \bar{B}_tu_t + \bar{F}_t
\begin{bmatrix}
w_t^{\top} & e_t^{\top} & \Delta\hat\theta_t^{\top}
\end{bmatrix}^{\top}.
\end{equation}
During planning, \eqref{eq:augmented_system} allows us to model parameter-estimation errors and estimate shifts as bounded disturbances. We define augmented SLS gains $\Phis_{k,j}\in\mathbb{R}^{(\nx+\ntheta)\times(\nx+\ntheta)}$ and $\Phiu_{k,j}\in\mathbb{R}^{\nuu\times(\nx+\ntheta)}$ for the state and control responses, yielding
\begin{equation}
\begin{aligned}
\Phis_{j+1,j} = I_{\nx + \ntheta}, \qquad \Phis_{k+1,j} = \bar{A}_k \Phis_{k,j} + \bar{B}_k\Phiu_{k,j},\label{eq:sls_gain_prop}
\end{aligned}
\end{equation}
for $k>j$. 
Define the augmented state deviation $\Delta s_k:=s_k-p_k$ from the nominal augmented state $p_k$, and the control deviation $\Delta u_k:=u_k-v_k$. Then
\begin{equation}\label{eq:reachability_augmented}
\begin{bmatrix}
\Delta s_k\\ \Delta u_k
\end{bmatrix}
=\sum_{j=0}^{k-1}\Phit_{k,j}
\begin{bmatrix}
F_jw_j+E_je_j\\
\Delta\hat\theta_j
\end{bmatrix},
\end{equation}
where {\small$\Phit_{k,j} = \begin{bmatrix} {\Phis_{k,j}}^{\top} & {\Phiu_{k,j}}^{\top}\end{bmatrix}^{\top}$}. Treating $w_j$, $e_j$, and $\Delta\hat\theta_j$ as independent disturbances and applying their respective bounds \eqref{eq:param_error_w_bound} and \eqref{eq:param_shift_bound} would yield a conservative bound on \eqref{eq:reachability_augmented}. Instead, we observe that $e_j$ and $\Delta\hat\theta_j$ are correlated with the disturbance history $\{w_\ell\}_{\ell=0}^{j}$. Formally,
\begin{proposition}\label{prop:disturbance_expression}
Let $e_0 = G_0\xi_0$, where $\xi_0 \in \mathcal{B}^{p_0}_\infty$ represents the initial parameter-error realization, with $p_0$ generators, and define the Kalman gain product
\begin{equation}
\Psi_{j,\ell} =
\begin{cases}
\textstyle \prod_{m=\ell}^{j-1}
(I_{\ntheta}-K_m\Pi_m), & j>\ell,\\
I_{\ntheta}, & j=\ell,\\
0_{\ntheta\times\ntheta}, & \text{otherwise}.
\end{cases}
\end{equation}
Then the parameter error $e_j$ and estimate shift $\Delta\hat{\theta}_j$ (Eqs. \eqref{eq:param_error_w_bound}-\eqref{eq:param_shift}) can be expressed in terms of the initial parameter error and disturbances $w_\ell$ for $\ell\le j$:
\begin{align}
e_j
&= \Psi_{j,0}G_0\xi_0
-\textstyle\sum_{\ell=0}^{j-1}
\Psi_{j,\ell+1}K_\ell W_\ell w_\ell,
\label{eq:param_error_dist}\\
\Delta\hat{\theta}_j
&= K_j\Pi_j\Psi_{j,0}G_0\xi_0
+K_jW_jw_j
\nonumber\\
&\qquad
-\textstyle\sum_{\ell=0}^{j-1}
K_j\Pi_j\Psi_{j,\ell+1}K_\ell W_\ell w_\ell.
\label{eq:param_shift_dist}
\end{align}
\end{proposition}
The proof of Prop. \ref{prop:disturbance_expression} is in App. \ref{app:proof_disturbance_expression}.
Substituting \eqref{eq:param_error_dist} and \eqref{eq:param_shift_dist} into \eqref{eq:reachability_augmented} expresses the augmented state and control deviations directly in terms of the process disturbances $w_j$. Formally,
\begin{theorem}\label{thm:adaptive_gains}
Define the adaptive disturbance gains
\begin{align}
\hspace{-7pt}\Mparam{k}
&=
\textstyle\sum_{j=0}^{k-1}
\left(
\TPhit_{k,j}E_j
+\BPhit_{k,j}K_j\Pi_j
\right)\Psit_{j,0}G_0, \\
\hspace{-7pt}\Mdist{k,j}
&=\TPhit_{k,j}F_j
+
\BPhit_{k,j}K_jW_j
\nonumber\\&-\textstyle\sum_{t=j+1}^{k-1}
\left(
\TPhit_{k,t}E_t
+\BPhit_{k,t}K_t\Pi_t
\right)\Psit_{t,j+1}K_jW_j,
\end{align}
where $\Phit_{k,j}:=\begin{bmatrix}\TPhit_{k,j}&\BPhit_{k,j}\end{bmatrix}$, with $\TPhit_{k,j} \in\mathbb{R}^{(\nx+\ntheta+\nuu)\times\nx}$ and $\BPhit_{k,j} \in\mathbb{R}^{(\nx+\ntheta+\nuu)\times\ntheta}$. Then the state and control deviations satisfy
\begin{equation}\label{eq:adaptive_state_disturbance}
\begin{bmatrix}
\Delta s_k\\ \Delta u_k
\end{bmatrix}
=\Mparam{k}\xi_0+\textstyle\sum_{j=0}^{k-1}\Mdist{k,j}w_j.
\end{equation}
\end{theorem}
The proof of Thm. \ref{thm:adaptive_gains} is in App. \ref{app:proof_adaptive_gains}. 
Notably, Thm.~\ref{thm:adaptive_gains} allows us to compute both robust reachable sets and an adaptive feedback controller from the SLS responses. Specifically, since $\xi_0\in\mathcal{B}^{p_0}_\infty$ and $w_j\in\mathcal{B}^{\nx}_\infty$, \eqref{eq:adaptive_state_disturbance} yields the following zonotopic reachable set for the augmented state and control:
\begin{equation}\label{eq:adaptive_reachable_set}\small
\begin{aligned}
\begin{bmatrix}s_k\\u_k\end{bmatrix}
&\in \mathcal{R}_k :=
\left\{\begin{bmatrix}p_k\\v_k\end{bmatrix}\right\}
\oplus \Mparam{k}\mathcal{B}^{p_0}_\infty
\oplus
\bigoplus_{j=0}^{k-1}
\Mdist{k,j}\mathcal{B}^{\nx}_\infty.\hspace{-5pt}
\end{aligned}
\end{equation}
Crucially, the adaptive disturbance gains $\Mparam{k}$ preserve correlations between parameter errors, estimate updates, and process disturbances, avoiding the conservatism of bounding them independently. 
To recover the adaptive controller, we assemble the optimized SLS responses into block lower-triangular matrices $\Phis$ and $\Phiu$, such that $[\Phis]_{k,j}=\Phis_{k,j}$ and $[\Phiu]_{k,j}=\Phiu_{k,j}$. Since $\Phis_{j+1,j}=I_{\nx+\ntheta}$, $\Phis$ has identity diagonal blocks and hence is causally invertible. Thus, following the standard SLS realization \cite{anderson2019system}, the feedback gains are
\begin{equation}\label{eq:adaptive_feedback_gain}
\mathcal{K}=\Phiu(\Phis)^{-1}.
\end{equation}
Denoting the $(k,j)$ block of $\mathcal{K}$ by $\mathcal{K}_{k,j}=[\mathcal{K}]_{k,j}$, with $\mathcal{K}_{k,j}=0$ for $j>k$, we write the associated causal adaptive policy as
\begin{equation}\label{eq:adaptive_feedback_policy}
u_k=v_k+\textstyle\sum_{j=1}^{k}
\mathcal{K}_{k,j}(s_j-p_j),
\end{equation}
where $p_j=[z_j^\top,\hat\theta_0^\top]^\top$ and
$s_j=[x_j^\top,\hat\theta_j^\top]^\top$.
Thus, the controller responds to both physical-state deviations and online parameter-estimate updates, while the corresponding reachable sets enable robust constraint enforcement.

\subsection{Robust Constraint Satisfaction}\label{subsec:robust_constraints}
To ensure robust constraint satisfaction, we require $g_k(x_k,u_k)\leq 0_{n_c}$ for all $k\in[N]$ and $g_N(x_N)\leq 0_{n_c^f}$ for every admissible $\xi_0$ and disturbance realization $w_j$. We achieve this by inflating the constraints of the nominal trajectory optimization to account for the adaptive closed-loop reachable set \eqref{eq:adaptive_reachable_set}. First, we linearize the constraints along the nominal trajectory:
{
\begin{equation}
\begin{aligned}
\hspace{-4pt}\textstyle \Gamma_k
&\textstyle:=\frac{\partial g_k}{\partial x}\big|_{(z_k,v_k)}, \, 
\Gamma_N
\textstyle:=\frac{\partial g_N}{\partial x}\big|_{z_N},\, \ 
\Lambda_k
\textstyle:=\frac{\partial g_k}{\partial u}\big|_{(z_k,v_k)}.\hspace{-12pt}
\end{aligned}
\end{equation}}

\noindent \looseness-1To construct constraint matrices that are dimensionally compatible with $s_k$, we augment the constraint Jacobians with zero blocks for the parameter estimate indices, as the physical state and input constraints do not directly depend on $\hat\theta_k$: 
\begin{equation}\label{eq:augmented_constraint_block}
\begin{aligned}
P_k &:= \begin{bmatrix}\Gamma_k & 0_{n_c \times \ntheta} & \Lambda_k\end{bmatrix}, \quad \forall k\in [N], \\
P_N &:= \begin{bmatrix}\Gamma_N & 0_{n_c^f \times (\ntheta + \nuu)}\end{bmatrix}.
\end{aligned}
\end{equation}
Using \eqref{eq:augmented_constraint_block} and the adaptive gains \eqref{eq:adaptive_state_disturbance}, robust constraint satisfaction follows if the following condition holds:
\begin{corollary}\label{cor:robust_constraint_satisfaction}
Robust constraint satisfaction is ensured if, for all $k \in [N]$,
\begin{equation}\label{eq:robust_constraint_combined}
g_k(z_k, v_k) + \beta_k \le 0_{n_c}, \quad
g_N(z_N) + \beta_N \le 0_{n_c^f},
\end{equation}
where, for all $k \in [N+1]$,
\begin{equation}\label{eq:beta_definition}
\beta_k := \|P_k\Mparam{k}\|_{1,r}
+\textstyle\sum_{j=0}^{k-1}\|P_k\Mdist{k,j}\|_{1,r}
+b^\textrm{lin}_k,
\end{equation}
and $b_k^\textrm{lin}$ is the correction factor from \cite[Eq. 22]{zhan2026robustly} accounting for nonlinear constraints.
\end{corollary}
The proof of Cor. \ref{cor:robust_constraint_satisfaction} is available in App. \ref{app:proof_robust_constraint_satisfaction}.

\subsection{Full Optimization Problem for Nonlinear Dynamics}\label{subsec:overall_optimization}

To extend our method to nonlinear dynamics, we augment the LTV-based analysis of Secs. \ref{subsec:method_sls}–\ref{subsec:robust_constraints} with an additional exogenous disturbance channel that bounds the worst-case linearization error, and apply the gain propagation \eqref{eq:sls_gain_prop} to the error dynamics of the nonlinear system, enabling us to account for the affine offset introduced by linearization while robustly bounding the resulting state and input deviations.
Using the initial parameter estimate $\hat\theta_0$, we linearize and obtain
{
    \begin{equation}
    \begin{aligned}
\textstyle A_k := \frac{\partial f}{\partial x}\big|_{(z_k,v_k, \hat \theta_0)},\quad  B_k := \frac{\partial f}{\partial u}\big|_{(z_k,v_k, \hat \theta_0)}, \\ \textstyle E_k := \frac{\partial f}{\partial \theta}\big|_{(z_k,v_k, \hat\theta_0)}.
\end{aligned}
\end{equation}
}

\noindent Following \cite{fang2026gpu}, we construct a valid disturbance scaling matrix $F_k \in \mathbb{R}^{\nx \times 2\nx}$ as a function of $F(\cdot)$, the state and control deviations $\Delta x_k$ and $\Delta u_k$, and the initial parameter set $\Theta_0$. The corresponding disturbance is denoted by $w_k \in \mathbb{R}^{2\nx}$. The LEB must remain valid over all admissible $\theta\in\Theta_0$, $\Delta x_k$, and $\Delta u_k$, with the deviations bounded using Cor. \ref{cor:robust_constraint_satisfaction}. Thus, $F_k$ also depends on the adaptive gains $\Mparam{k}$ and $\Mdist{k,j}$ for $j<k$. We write this over-bound compactly as
\begin{equation}\label{eq:linearization_error}
F_k = \zeta\left(z_k,v_k, F(\cdot), \Mparam{k}, \{\Mdist{k,j}\}_{j=0}^{k-1}, \Theta_0\right),
\end{equation}
where $\zeta$ is described in \cite[Eq. 23]{fang2026gpu}. The corresponding disturbance is $w_k \in \mathbb{R}^{2\nx}$, allowing us to reconstruct \eqref{eq:augment_terms}-\eqref{eq:adaptive_reachable_set} with guarantees for nonlinear systems. 
Using the adaptive gains from Thm. \ref{thm:adaptive_gains}, the RANOCP becomes
\begin{subequations}\label{eq:fastsls_problem}
    \begin{align}
        \hspace{-13pt}\min_{\substack{
        \mathbf{z},\, \mathbf{v}, \mathbf{\Phi}
        }} \ \
        & J(\mathbf{z},\mathbf{v}) + \tilde{H}_0(\Phit) \label{eq:sls_problem_obj} \\[-1mm]
        \text{s.t.}\quad \
        & z_{k+1} = f(z_k, v_k; \hat\theta_0), \quad \forall k \in [N],\quad z_0 = \bar{x}_0,\label{eq:sls_problem_dyn} \\[-0.5mm] 
        & \Phis_{k+1, j} = \bar{A}_k\Phis_{k, j} + \bar{B}_k\Phiu_{k,j},\ \ \Phis_{j+1,j} = I_{\nx + \ntheta}, \nonumber\\[-0.5mm]
        &  \quad\quad\forall j \in [N], \hspace{5pt} \forall k \in [j+1, N - 1], \label{eq:sls_constraint}\\[-0.5mm]
        & g_k(z_k, v_k) + \beta_k \le  0_{n_c}, \qquad \forall k \in [N],\label{eq:rmpc_stage_tightening}\\[-0.5mm]
        & g_N(z_N)+ \beta_N \le 0_{n_c^f}\label{eq:rmpc_term_tightening},
        \vspace{-2pt}
    \end{align}
\end{subequations}
where $\mathbf{z}:=\{z_k\}_{k=0}^{N}$ and $\mathbf{v}:=\{v_k\}_{k=0}^{N-1}$. Equation \eqref{eq:sls_problem_dyn} enforces the nominal dynamics and initial condition, \eqref{eq:sls_constraint} imposes the system-level constraints, and \eqref{eq:rmpc_stage_tightening}-\eqref{eq:rmpc_term_tightening} enforce robust constraint satisfaction via \eqref{eq:robust_constraint_combined}-\eqref{eq:beta_definition}. The objective \eqref{eq:sls_problem_obj} combines the tracking cost $J(\mathbf{z},\mathbf{v})$ with the tube penalty $\tilde{H}_0(\Phit)$ from \cite[Eqs. 49, 52-53]{srinivasan2026safety}. Solving \eqref{eq:fastsls_problem} over a fixed planning horizon yields a robust adaptive trajectory optimization policy, while repeatedly solving it in receding-horizon fashion yields a RAMPC controller.

\begin{remark}[Recursive feasibility]
The resulting A-SLS controller is compatible with the recursive-feasibility construction of \cite{leeman2025guaranteed}. In particular, \eqref{eq:fastsls_problem} can be augmented with a robust positively invariant terminal set and associated terminal controller, together with the shifted nominal and auxiliary trajectory construction in \cite[Eqs. 27-29]{leeman2025guaranteed}. The adaptive disturbance representation in \eqref{eq:adaptive_state_disturbance} and nonlinear linearization-error bound in \eqref{eq:linearization_error} can then be treated as the bounded prediction uncertainty. Under the corresponding assumptions, initial feasibility implies recursive feasibility and robust constraint satisfaction. 
\end{remark}

\section{Results}
We evaluate A-SLS as both a robust trajectory optimization method in Sec. \ref{sec:rob_traj_opt} and a RAMPC policy in Sec. \ref{sec:rob_ada_mpc}. Our solver builds on \cite{fang2026safe}. All simulations use an NVIDIA RTX 4090 GPU (24 GB VRAM), while the hardware experiment uses an Apple M4 Pro CPU (12 cores, 24 GB RAM). We report a log-volume surrogate for the forecast parameter uncertainty sets,
\begin{equation}\label{eq:param_log_volume}
    \log V_{\theta,k}
    := \textstyle\sum_{i=1}^{n_\theta}
    \log \|\bar{e}_i^{(n_\theta)} G_k\|_1,
\end{equation}
and a log-volume surrogate for the positional tubes,
\begin{equation}\label{eq:pos_log_volume}
\log V_{p,k}
:=
\textstyle\sum_{i=1}^{n_p}
\log(
    \|\bar{e}_i^{(n_a)} \Mparam{k}\|_1
    +\sum_{j=0}^{k-1}
    \|\bar{e}_i^{(n_a)} \Mdist{k,j}\|_1
),
\end{equation}
where $n_a=\nx+\ntheta+\nuu$, $n_p=2$ for Dubins car and the planar quadrotor, and $n_p=3$ otherwise.

\subsection{Robust Trajectory Optimization}\label{sec:rob_traj_opt}
\noindent\textbf{Dubins Ablation:} We evaluate two Dubins car models from \cite[Eq. 39]{srinivasan2026safety} ($n_x=4$, $n_u=2$): one with matched uncertainty and one with unmatched uncertainty, each with $n_\theta=2$ and $\theta\in\Theta_0=0.0075\mathcal{B}^2_\infty$. We denote uncertainty as matched if it enters the dynamics through the same channels as the control inputs (i.e., lies in the column space of $B_t$), and unmatched otherwise \cite{sinha2022matched}. For the discretized dynamics with timestep $dt$, the matched and unmatched parameter sensitivities are
$E_k:=dt\begin{bmatrix}0_{2\times2} & I_2\end{bmatrix}^{\top}$
and
$E_k:=dt\begin{bmatrix}I_2 & 0_{2\times2}\end{bmatrix}^{\top}$,
respectively, with $F_k=0.0003I_4$ in both cases. We compare A-SLS with and without LEBs against a non-adaptive SLS formulation. The car moves from $p_{xy}:=(p_x,p_y)=(1.0,1.0)\,\textrm{m}$, with $\hat\theta_0=0_2$, to $(1.0,-1.0)\,\textrm{m}$ while avoiding an obstacle centered at $c_\textrm{obs}=(0.95,0.0)\,\textrm{m}$ with radius $r_\textrm{obs}=0.35\,\textrm{m}$. We use horizon $N=70$ and timestep $dt=0.05\,\textrm{s}$.
For each ablation, we run 10 adversarial-disturbance rollouts with a fixed worst-case disturbance primitive $w_k=w\in\{1,-1\}^{4}$ for all $k\in[N]$ and the true parameter selected from a corner of $\Theta_0$, together with 30 random rollouts where $w_k\sim U[-1,1]$ and $\theta$ is sampled uniformly from $\Theta_0$. For the matched and unmatched cases, A-SLS with LEBs reduces the log positional tube volume by $19\%$ and $11\%$, shrinking the physical position-tube volume to $18\%$ and $53\%$ of the baseline SLS volume, respectively, as shown in Tab. \ref{tab:combined_rollout_results}. Online parameter learning also reduces tracking error. Removing the LEBs further reduces tube volume but sacrifices robustness to linearization error, producing tube-containment rates below $100\%$ in the unmatched case. Overall, these results show that A-SLS reduces conservativeness and tracking error under matched and unmatched uncertainty, while LEBs ensure robust containment under nonlinear dynamics.

\noindent\textbf{Dubins 2-Obstacle Baseline:} We compare A-SLS with LEBs against an adaptive control contraction metric (CCM) method~\cite{kohler2021robust} and an adaptive disturbance-feedback (DF) method without system-level constraints~\cite{sinha2022matched} on the unmatched Dubins dynamics. The car moves from $(0.0,1.0)\,\textrm{m}$ to $(0.25,-1.0)\,\textrm{m}$ while avoiding obstacles centered at $(-0.25,0.20)\,\textrm{m}$ and $(0.25,-0.25)\,\textrm{m}$, each with radius $0.23\,\textrm{m}$. We use $N=75$, $dt=0.05\,\textrm{s}$, 10 adversarial rollouts, and 30 random rollouts. Across rollouts, A-SLS produces smaller positional tubes (Tab. \ref{tab:combined_rollout_results}), enabling a shorter $s$-curve trajectory in Fig. \ref{fig:Dubins_Rollout}(a). The DF approach omits linearization errors, causing adversarial rollouts to leave its tubes despite those tubes being larger than A-SLS (Fig. \ref{fig:Dubins_Rollout}(b-c)). Overall, these results demonstrate that accounting for future parameter learning reduces conservativeness relative to CCM-based methods, while LEBs are essential for guaranteeing robust containment under nonlinear dynamics.

\begin{figure}[t]
    \centering
    \includegraphics[width=\linewidth]{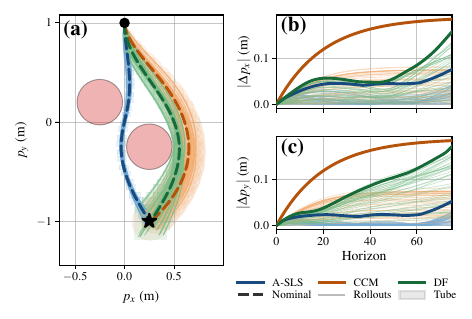}
    \vspace{-25pt}
    \caption{\looseness-1\textbf{Dubins Unmatched Rollout.} (a) Planned trajectories (dashed lines) for each approach, with position tubes and sample rollouts. (b-c) The $p_x$ and $p_y$ tubes over all adversarial and random rollouts. The tighter tubes enable A-SLS to plan a shorter S-curve.}
    \label{fig:Dubins_Rollout}
    \vspace{-12pt}
\end{figure}

\noindent\textbf{Quadrotor Active Information Gathering:} We evaluate A-SLS on trajectory optimization for the 12-state, 4-input quadrotor model from \cite[Eq. 2.19]{sabatino2015quadrotor}, with $n_\theta=3$ unknown unmatched wind-disturbance parameters entering through
$E_k = 0.003\begin{bmatrix}
0_{3\times 6} & I_3 & 0_{3\times3}
\end{bmatrix}^{\top},$
with $\theta\in\Theta_0=0.125\mathcal{B}^3_\infty$ and $\hat\theta_0=0_3$. We use the altitude-varying disturbance
$F(x_k)=0.0009(1-\tanh(15(\bar{e}_3^{(12)}x_k-0.05)))I_{12}$, 
where $\bar{e}_3^{(12)}$ extracts the $p_z$ coordinate. To encourage informative trajectories, we augment the nominal cost with
\begin{equation}\label{eq:info_gather}
J_\textrm{info}(\mathbf{z}) = \lambda\textstyle\sum_{k=1}^{N-1}\|W_k\|_{F},
\end{equation}
using $\lambda=600$, where $W_k$ from \eqref{eq:param_uncertainty_definition} depends implicitly on the nominal state $z_k$ through state-dependent parameter sensitivities $E_k$ and disturbances $F_k$.

The quadrotor moves from $p_{xyz}:=(p_x,p_y,p_z)=(-0.75,-0.75,0.40)\,\textrm{m}$ to $(1.00,0.80,0.40)\,\textrm{m}$ while avoiding a cylindrical obstacle centered at $c_\textrm{obs}=(0.10,0.10)\,\textrm{m}$ with radius $r_\textrm{obs}=0.25\,\textrm{m}$. We use $N=110$, $dt=0.03\,\textrm{s}$, and compare A-SLS+, which uses \eqref{eq:info_gather}, against A-SLS without the information-gathering cost over 120 random and 80 adversarial rollouts. Without $J_\textrm{info}$, the nominal trajectory remains near the goal elevation (Fig. \ref{fig:quadrotor_active_info}(b)), yielding larger positional tubes (Tab. \ref{tab:combined_rollout_results} and Fig. \ref{fig:quadrotor_active_info}(a-b)). By minimizing $J_\textrm{info}$, A-SLS+ seeks higher, lower-disturbance altitudes, reducing parameter uncertainty (Tab. \ref{tab:combined_rollout_results}). The smaller $p_z$ tube tightens the measurement-error bound $W_k$ through reduced linearization error, which further shrinks parameter uncertainty and reinforces tube contraction (Fig. \ref{fig:quadrotor_active_info}(c)). Overall, these results show that A-SLS enables active information gathering by planning trajectories that accelerate parameter learning.

\begin{figure}[t]
    \centering
    \includegraphics[width=\linewidth]{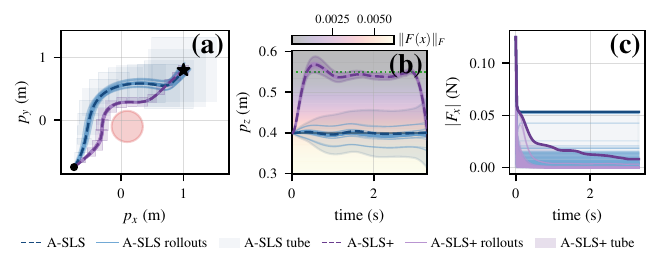}
    \vspace{-25pt}
    \caption{\looseness-1\textbf{12-State Quadrotor Active Information Gathering.} (a-b) Planned quadrotor trajectories (dashed lines) using A-SLS and A-SLS+ (with $J_\textrm{info}$), together with simulated rollouts. (c) Parameter uncertainty over time, with measured parameter errors from the rollouts. A-SLS+ further reduces parameter uncertainty by reaching higher altitudes with lower disturbance (b).}
    \label{fig:quadrotor_active_info}
    \vspace{-20pt}
\end{figure}

\begin{table*}[!t]
\centering
\small
\setlength{\tabcolsep}{3pt}
\medmuskip=1mu
\thickmuskip=2mu
\caption{\textbf{Full-Horizon Rollouts.} Metrics from trajectory optimization experiments (computed per step, except for goal distance).}
\label{tab:combined_rollout_results}
\resizebox{\textwidth}{!}{%
\begin{tabular}{@{}llcccccc@{}}
\toprule
Experiment & Method & Rand. contain. (\%) & Adv contain. (\%) & Goal dist. (m) & Tracking RMSE (m) & Log pos. tube volume & Log param. tube volume \\
\midrule
\multirow{4}{*}{\shortstack[l]{Dubins Matched\\Ablation}} & A-SLS & {\small $\bm{100}$} & {\small $\bm{100}$} & {\small $0.00376 \pm 0.0015$} & {\small $\bm{0.00116 \pm 0.00073}$} & {\small $-10.6 \pm 1.1$} & {\small $\bm{-10.1 \pm 0.6}$} \\
 & A-SLS w/o LEB & {\small $\bm{100}$} & {\small $\bm{100}$} & {\small $\bm{0.00344 \pm 0.0015}$} & {\small $0.00118 \pm 0.00073$} & {\small $\bm{-10.6 \pm 1.1}$} & {\small $\bm{-10.1 \pm 0.6}$} \\
 & SLS & {\small $\bm{100}$} & {\small $\bm{100}$} & {\small $0.0124 \pm 0.0057$} & {\small $0.00752 \pm 0.0042$} & {\small $-8.92 \pm 1.6$} & {\small --} \\
 & SLS w/o LEB & {\small $\bm{100}$} & {\small $\bm{100}$} & {\small $0.0122 \pm 0.0058$} & {\small $0.00762 \pm 0.0043$} & {\small $-8.94 \pm 1.6$} & {\small --} \\
\midrule
\multirow{4}{*}{\shortstack[l]{Dubins Unmatched\\Ablation}} & A-SLS & {\small $\bm{100}$} & {\small $\bm{100}$} & {\small $\bm{0.0225 \pm 0.0075}$} & {\small $0.0199 \pm 0.012$} & {\small $-6.23 \pm 1$} & {\small $-10.1 \pm 0.59$} \\
 & A-SLS w/o LEB & {\small $99.7$} & {\small $97.8$} & {\small $0.0235 \pm 0.0081$} & {\small $\bm{0.0182 \pm 0.012}$} & {\small $\bm{-7.8 \pm 0.8}$} & {\small $\bm{-10.1 \pm 0.6}$} \\
 & SLS & {\small $\bm{100}$} & {\small $\bm{100}$} & {\small $0.0542 \pm 0.017$} & {\small $0.0509 \pm 0.022$} & {\small $-5.6 \pm 1.2$} & {\small --} \\
 & SLS w/o LEB & {\small $99.8$} & {\small $97.8$} & {\small $0.0504 \pm 0.021$} & {\small $0.0485 \pm 0.023$} & {\small $-5.94 \pm 1$} & {\small --} \\
\midrule
\multirow{3}{*}{\shortstack[l]{Dubins 2-Obstacle\\Baselines}} & A-SLS & {\small $\bm{100}$} & {\small $\bm{100}$} & {\small $\bm{0.0281 \pm 0.014}$} & {\small $\bm{0.0142 \pm 0.0087}$} & {\small $\bm{-6.99 \pm 0.8}$} & {\small --} \\
 & CCM & {\small $\bm{100}$} & {\small $\bm{100}$} & {\small $0.0726 \pm 0.028$} & {\small $0.0637 \pm 0.031$} & {\small $-4.13 \pm 1.1$} & {\small --} \\
 & DF & {\small $\bm{100}$} & {\small $95.5$} & {\small $0.128 \pm 0.05$} & {\small $0.0733 \pm 0.042$} & {\small $-5.63 \pm 1.3$} & {\small --} \\
\midrule
\multirow{2}{*}{\shortstack[l]{Quadrotor Info\\Gathering}} & A-SLS & {\small $\bm{100}$} & {\small $\bm{100}$} & {\small $0.0339 \pm 0.027$} & {\small $0.0245 \pm 0.018$} & {\small $-6.75 \pm 2.9$} & {\small $-8.77 \pm 0.24$} \\
 & A-SLS+ (w/ $J_\textrm{info}$) & {\small $\bm{100}$} & {\small $\bm{100}$} & {\small $\bm{0.0118 \pm 0.0035}$} & {\small $\bm{0.00897 \pm 0.0054}$} & {\small $\bm{-10.3 \pm 1.6}$} & {\small $\bm{-12.7 \pm 1.6}$} \\
\bottomrule
\end{tabular}
}
\vspace{-3pt}
\end{table*}

\subsection{Robust Adaptive MPC}\label{sec:rob_ada_mpc}
\noindent\textbf{Planar Quadrotor Baseline:} We compare A-SLS against the DF and CCM baselines in an MPC setting using the inverse-mass parameterized planar quadrotor from \cite[Sec. 4]{sasfi2023robust} ($n_x=6$, $n_u=2$, $n_\theta=1$), with disturbance model $F(x_k)=0.005
\begin{bmatrix}
{0}_{1\times 3} \hspace{-1mm}& \hspace{-1mm}\cos(\bar{e}_3^{(6)}x_k) \hspace{-1mm}&\hspace{-1mm} -\sin(\bar{e}_3^{(6)}x_k) \hspace{-1mm}&\hspace{-1mm} 0
\end{bmatrix}^{\top}\bar{e}_1^{(6)},
$
where $\bar{e}_3^{(6)}x_k$ is the quadrotor angle and $\bar{e}_1^{(6)}$ ensures $F(x_k)\in\mathbb{R}^{\nx\times\nx}$. We use $\hat\theta_0=2.568$ with $\theta\in\Theta_0=[0.99\hat\theta_0,1.01\hat\theta_0]$. The quadrotor moves from $p_{xy}=(0,0)\,\textrm{m}$ to $(2,0)\,\textrm{m}$ while avoiding a circular obstacle centered at $c_\textrm{obs}=(-0.1,1.0)\,\textrm{m}$ with radius $r_\textrm{obs}=0.5\,\textrm{m}$. All solvers use $N=10$, $dt=0.05\,\textrm{s}$, and up to $H=40$ MPC steps to reach the goal.

All approaches reach the goal within the allotted horizon, but A-SLS is the only method with solve times below the real-time budget ($<dt$) across all 40 runs (Tab. \ref{tab:combined_mpc_results}). A-SLS also achieves the lowest control effort, measured by $\|u-u_\textrm{hover}\|$. Baselines using set-membership estimation (SME), which enforce $\Theta_t\subseteq\Theta_{t-1}$ rather than bounded parameter updates, achieve smaller terminal parameter sets but larger positional tubes. 
Specifically, for DF~\cite{sinha2022matched}, SME recursively refines the feasible parameter set using observed state transitions and bounded disturbances, with uncertainty quantified by its Chebyshev center and radius, whereas for CCM~\cite{sasfi2023robust}, SME refines a polytopic uncertainty set by intersecting it with measurement-consistent parameter constraints.
Replacing A-SLS parameter measurements with SME yields parameter errors comparable to the baselines while retaining smaller tube volumes (Tab. \ref{tab:combined_mpc_results}). Overall, these results show A-SLS achieves real-time robust adaptive MPC with much faster solve times and lower control effort than  baselines, while maintaining tight reachable tubes.

\begin{table*}[!htbp]
\centering
\small
\setlength{\tabcolsep}{3pt}
\medmuskip=1mu
\thickmuskip=2mu
\caption{\textbf{Robust MPC.} Metrics from robust MPC with online replanning (computed per step, except for log final $\theta$-tube volume).}
\label{tab:combined_mpc_results}
\resizebox{\textwidth}{!}{%
\begin{tabular}{@{}llcccccc@{}}
\toprule
Experiment & Method & Solve $Q_1$/Med./$Q_3$ (ms) & Tracking RMSE (m) & Log forecast pos. tube volume & Control effort & Param. error & Log final $\theta$-tube volume \\
\midrule
\multirow{4}{*}{\shortstack[l]{Planar\\Quadrotor}} & A-SLS & {\small $5.16/5.46/5.83$} & {\small $7.3e-05 \pm 3.6e-05$} & {\small $-13.7 \pm 2.5$} & {\small $\bm{0.127 \pm 0.22}$} & {\small $0.000167 \pm 0.00013$} & {\small $-6.64 \pm 0.033$} \\
 & A-SLS SME & {\small $\bm{5.15/5.43/5.78}$} & {\small $7.3e-05 \pm 3.6e-05$} & {\small $\bm{-14.1 \pm 2.6}$} & {\small $0.127 \pm 0.22$} & {\small $7.65e-05 \pm 6.9e-05$} & {\small $-8.92 \pm 0.68$} \\
 & CCM & {\small $764/887/1.1e+03$} & {\small $\bm{7.24e-05 \pm 3.6e-05}$} & {\small $-6.57 \pm 1.3$} & {\small $0.343 \pm 0.59$} & {\small $3.97e-05 \pm 3e-05$} & {\small $-9.51 \pm 0.65$} \\
 & DF  & {\small $178/188/207$} & {\small $\bm{7.24e-05 \pm 3.6e-05}$} & {\small $-13.7 \pm 3.6$} & {\small $0.352 \pm 0.62$} & {\small $\bm{3.76e-05 \pm 4.2e-05}$} & {\small $\bm{-9.83 \pm 0.9}$} \\
\midrule
\multirow{2}{*}{\shortstack[l]{Quadruped\\Damping}} & A-SLS & {\small $\bm{15.7/16.1/17.4}$} & {\small $\bm{0.000617 \pm 0.00041}$} & {\small $\bm{-6.97 \pm 1.7}$} & {\small --} & {\small $\bm{0.499 \pm 0.25}$} & {\small $\bm{-14.5 \pm 2.4}$} \\
 & SLS & {\small $14.6/17.4/37.4$} & {\small $0.00174 \pm 0.0014$} & {\small $-6.06 \pm 1.5$} & {\small --} & {\small $2.47 \pm 0.36$} & {\small $3.76 \pm 0$} \\
\midrule
\shortstack[l]{Crazyflie\\Hardware} & A-SLS & {\small $\bm{8.02/8.15/8.35}$} & {\small $\bm{0.000941 \pm 0.00058}$} & {\small $\bm{-10.6 \pm 2.7}$} & {\small $\bm{0.63 \pm 0.058}$} & {\small $\bm{0.293 \pm 0.23}$} & {\small $\bm{0.119 \pm 0.0017}$} \\
\bottomrule
\end{tabular}
}
\vspace{-15pt}
\end{table*}

\noindent\textbf{Quadruped Scalability:} To evaluate scalability, we consider the whole-body quadruped model from \cite[App. E]{fang2026safe} with damped or damaged joints ($n_x=61$, $n_u=12$, $n_\theta=8$). For each uncertain joint,
\begin{equation}
    \tau_{\textrm{eff}}^j = u^j - \theta^j\dot{q}^j,
\end{equation}
where $\tau_{\textrm{eff}}^j$ is the effective torque for joint $j$, $u^j$ is the commanded joint input, $\theta^j$ is the damping, and $\dot{q}^j$ is the joint's velocity. We assume no initial joint damage, $\hat\theta_0=0_8$, for the four calf and four thigh joints, with $\theta\in\Theta_0=[-1.6,1.6]^8$. Disturbance bounds are $0.004$ for base $p_x,p_y$, $0.020$ for thigh joints, and $0.025$ for calf joints. To model real-world mismatch, the true plant additionally includes model error exceeding these assumed disturbance bounds. We use the measurement gate $C_t$ \eqref{eq:gated_information} to admit joint-position information for a parameter only when the corresponding leg is in swing. The task requires stable motion for $H=500$ steps while tracking $v_x=0.15$ and avoiding an inflated circular obstacle centered at $c_\textrm{obs}=(1.0,0.25)\,\textrm{m}$ with radius $r_\textrm{obs}=0.42\,\textrm{m}$. Both solvers use $N=14$ and $dt=0.032\,\textrm{s}$.

We compare A-SLS against SLS, the only baseline that scaled to this system. A-SLS succeeds in $38/40$ runs while maintaining safe obstacle clearance, with an example shown in Fig. \ref{fig:quadruped_rollout_figure}(a); SLS remains stable in only $24/40$ runs. Despite violations of the assumed disturbance bounds, the parameter estimates and uncertainty sets remain near the true parameters (Fig. \ref{fig:quadruped_rollout_figure}(b-c)) using \eqref{eq:gain_update_law}. In contrast, SME requires hard set-containment guarantees, so violations of the assumed disturbance bounds can invalidate its parameter-set update. For non-adaptive SLS, the third-quartile solve time exceeds the real-time budget (Tab. \ref{tab:combined_mpc_results}) because the large initial parameter uncertainty and larger tracking errors require more solver iterations. Nevertheless, A-SLS maintains stable legged motion with a median solve time of $16.1\,\textrm{ms}$, remaining within the real-time budget of $32\,\textrm{ms}$. Overall, these results show that A-SLS scales to high-dimensional whole-body control, achieving improved stability and tracking performance while meeting real-time computational constraints.
\noindent\textbf{Hardware Validation:} Lastly, we validate A-SLS on a Crazyflie with an unknown thrust-to-acceleration parameter ($n_\theta=1$), which captures the combined effects of mass, battery level, and motor health. We use the 12-state, 4-input quadrotor model ($n_x=12$, $n_u=4$), replacing the inverse mass $1/m$ with the uncertain parameter $\theta$. We set $\hat\theta_0=13.5$ and $\Theta_0=[11.5,15.5]$; prior to evaluation, we measured $\theta=12.26$ with a vehicle mass of $35\,\textrm{g}$. Disturbance bounds are $0.004$ for position, $0.008$ for attitude, $0.16$ for angular rates, and $0.04$, $0.016$, and $0.025$ for the velocities $v_x$, $v_y$, and $v_z$, respectively. We use $N=10$ and $dt=0.05\,\textrm{s}$. The task is to move $2\,\textrm{m}$ in $p_y$ while increasing altitude by $0.2\,\textrm{m}$ in $p_z$, subject to a normalized thrust limit of $0.95$ ($1$ is the hardware limit). Because $v_z$ is reliably measured, the gate in \eqref{eq:gated_information} uses only $v_z$ residuals for parameter estimation.

Across 15 hardware runs, the Crazyflie learns the parameter and stabilizes to reach the goal (Fig. \ref{fig:front_figure_quadrotor_hardware.pdf}(b)). Since $\hat\theta_0>\theta$, the anticipated thrust was below the hover-required thrust and the vehicle initially loses altitude. However, A-SLS recovers the altitude as the parameter estimate converges around $0.55\,\textrm{s}$ (Fig. \ref{fig:front_figure_quadrotor_hardware.pdf}(a,c)). A-SLS also maintains real-time performance below the $50\,\textrm{ms}$ budget with a median solve time of $8.15\,\textrm{ms}$ while respecting the thrust limit (Fig. \ref{fig:front_figure_quadrotor_hardware.pdf}(d)). Overall, these results show that A-SLS enables real-time parameter learning and robust stabilization on physical hardware, recovering from inaccurate initial dynamics estimates while respecting actuator constraints.

\section{Conclusion}
We presented A-SLS, a robust adaptive control framework that accounts for guaranteed future parameter learning while preserving correlations between process disturbances, parameter error, and estimate updates to construct tighter reachable tubes. Across trajectory optimization and RAMPC experiments, A-SLS reduces conservatism and tracking error, supports active information gathering, and achieves real-time performance on nonlinear robotic systems, including whole-body quadruped control. Hardware experiments further demonstrate online parameter learning and stabilization on a Crazyflie with an incorrect initial parameter estimate. Future efforts will extend A-SLS to the output-feedback setting \cite{li2026robustness, leeman2026vision} to enable adaptation to uncertain perception.

\bibliographystyle{IEEEtran}
\bibliography{references}

\appendix

\subsection{Proof of Proposition~\ref{prop:disturbance_expression}}
\label{app:proof_disturbance_expression}
\noindent\emph{Proof.}
From the parameter measurement \eqref{eq:parameter_measure} and the center update
\eqref{eq:zonotope_center_update},
\begin{align}
\Delta\hat\theta_t
&=K_t\left(\widetilde\theta_{t+1}-\Pi_t\hat\theta_t\right) \\
&=K_t\Pi_t\left(\theta-\hat\theta_t\right)+K_tW_tw_t \\
&=K_t\Pi_t e_t+K_tW_tw_t. \label{eq:app_shift_one_step}
\end{align}
Since $e_t=\theta-\hat\theta_t$ and $\theta$ is constant over the horizon,
\begin{align}
e_{t+1}
&=\theta-\hat\theta_{t+1}
=e_t-\Delta\hat\theta_t \\
&=(I_{\ntheta}-K_t\Pi_t)e_t-K_tW_tw_t. \label{eq:app_error_recursion}
\end{align}
For clarity, the ordered product in $\Psit_{j,\ell}$ is understood as
\begin{equation}
\Psit_{j,\ell}
=(I_{\ntheta}-K_{j-1}\Pi_{j-1})\cdots
 (I_{\ntheta}-K_{\ell}\Pi_{\ell}), \qquad j>\ell.
\end{equation}
Repeated substitution of \eqref{eq:app_error_recursion}, starting from
$e_0=G_0\xi_0$, gives
\begin{equation}
 e_j
 =\Psit_{j,0}G_0\xi_0
 -\sum_{\ell=0}^{j-1}
 \Psit_{j,\ell+1}K_\ell W_\ell w_\ell,
\end{equation}
which is \eqref{eq:param_error_dist}. Finally, substituting this expression for
$e_j$ into \eqref{eq:app_shift_one_step} yields
\begin{align}
\Delta\hat\theta_j
&=K_j\Pi_j\Psit_{j,0}G_0\xi_0+K_jW_jw_j \nonumber\\
&\quad-\sum_{\ell=0}^{j-1}
K_j\Pi_j\Psit_{j,\ell+1}K_\ell W_\ell w_\ell,
\end{align}
which is \eqref{eq:param_shift_dist}. \hfill$\square$

\subsection{Proof of Theorem~\ref{thm:adaptive_gains}}
\label{app:proof_adaptive_gains}
\noindent\emph{Proof.}
Partition the augmented SLS response as
$\Phit_{k,j}=[\TPhit_{k,j}\;\BPhit_{k,j}]$, so that
\eqref{eq:reachability_augmented} becomes
\begin{align}
\begin{bmatrix}\Delta s_k\\ \Delta u_k\end{bmatrix}
&=\sum_{t=0}^{k-1}\TPhit_{k,t}(F_tw_t+E_te_t)
 +\sum_{t=0}^{k-1}\BPhit_{k,t}\Delta\hat\theta_t. \label{eq:app_sls_partition}
\end{align}
Using the one-step identity
$\Delta\hat\theta_t=K_t\Pi_te_t+K_tW_tw_t$ from
\eqref{eq:app_shift_one_step}, \eqref{eq:app_sls_partition} can be regrouped as
\begin{align}
\begin{bmatrix}\Delta s_k\\ \Delta u_k\end{bmatrix}
&=\sum_{t=0}^{k-1}
\left(\TPhit_{k,t}E_t+\BPhit_{k,t}K_t\Pi_t\right)e_t \nonumber\\
&\quad+\sum_{t=0}^{k-1}
\left(\TPhit_{k,t}F_t+\BPhit_{k,t}K_tW_t\right)w_t. \label{eq:app_sls_regroup}
\end{align}
Substituting \eqref{eq:param_error_dist} into \eqref{eq:app_sls_regroup}, the
coefficient multiplying the initial parameter realization $\xi_0$ is
\begin{equation}
\sum_{t=0}^{k-1}
\left(\TPhit_{k,t}E_t+\BPhit_{k,t}K_t\Pi_t\right)
\Psit_{t,0}G_0
=\Mparam{k}.
\end{equation}
Now fix a disturbance index $j<k$. The disturbance $w_j$ contributes directly
at $t=j$ through both the plant and the estimator, giving
\begin{equation}
\TPhit_{k,j}F_j+\BPhit_{k,j}K_jW_j.
\end{equation}
It also contributes indirectly to every later parameter error $e_t$, $t>j$,
through the term $-\Psit_{t,j+1}K_jW_jw_j$. Hence the total coefficient of
$w_j$ is
\begin{align}
&\TPhit_{k,j}F_j+\BPhit_{k,j}K_jW_j \nonumber\\
&\quad-\sum_{t=j+1}^{k-1}
\left(\TPhit_{k,t}E_t+\BPhit_{k,t}K_t\Pi_t\right)
\Psit_{t,j+1}K_jW_j
=\Mdist{k,j}.
\end{align}
Collecting the coefficients of $\xi_0$ and all $w_j$, $j=0,\ldots,k-1$,
therefore gives
\begin{equation}
\begin{bmatrix}\Delta s_k\\ \Delta u_k\end{bmatrix}
=\Mparam{k}\xi_0+\sum_{j=0}^{k-1}\Mdist{k,j}w_j,
\end{equation}
as claimed. \hfill$\square$

\subsection{Proof of Corollary~\ref{cor:robust_constraint_satisfaction}}
\label{app:proof_robust_constraint_satisfaction}
\noindent\emph{Proof.}
For each time $k$, Theorem~\ref{thm:adaptive_gains} gives the deviation
\begin{equation}
\delta_k:=
\begin{bmatrix}\Delta s_k\\ \Delta u_k\end{bmatrix}
=\Mparam{k}\xi_0+\sum_{j=0}^{k-1}\Mdist{k,j}w_j,
\end{equation}
where $\|\xi_0\|_\infty\le 1$ and $\|w_j\|_\infty\le 1$.
For any row vector $a^\top$, the dual-norm identity for the $\ell_\infty$ ball is
\begin{equation}
\sup_{\|z\|_\infty\le 1}a^\top z=\|a\|_1. \label{eq:app_dual_norm}
\end{equation}
Applying \eqref{eq:app_dual_norm} independently to the initial parameter
uncertainty and each process disturbance gives, row-wise,
\begin{align}
\sup_{\substack{\|\xi_0\|_\infty\le1\\
                 \|w_j\|_\infty\le1}}
P_k\delta_k
&=\|P_k\Mparam{k}\|_{1,r}
 +\sum_{j=0}^{k-1}\|P_k\Mdist{k,j}\|_{1,r}. \label{eq:app_support_stage}
\end{align}
Thus the nominal constraint value, plus the support function of the reachable
uncertainty set in each constraint direction and the linearization-error
backoff $b_k^{\textrm{lin}}$, is no greater than zero whenever
\begin{equation}
 g_k(z_k,v_k)+\|P_k\Mparam{k}\|_{1,r}
 +\sum_{j=0}^{k-1}\|P_k\Mdist{k,j}\|_{1,r}
 +b_k^{\textrm{lin}}\le 0_{n_c}.
\end{equation}
The same argument at the terminal time, with $P_N$ and the response at $N$,
yields
\begin{equation}
 g_N(z_N)+\|P_N\Mparam{N}\|_{1,r}
 +\sum_{j=0}^{N-1}\|P_N\Mdist{N,j}\|_{1,r}
 +b_N^{\textrm{lin}}\le 0_{n_c^f}.
\end{equation}
Therefore the stated tightenings imply robust satisfaction for every
admissible $\xi_0$ and disturbance sequence. \hfill$\square$

\end{document}